\documentclass[journal]{IEEEtran}
\usepackage{url}
\usepackage{amsfonts}
\usepackage{amsmath,amssymb,graphicx,epsfig}
\usepackage{enumerate,amsthm,epstopdf}
\usepackage{bbm}
\usepackage{dsfont}
\usepackage{algpseudocode}
\usepackage{algorithmicx}
\usepackage{algorithm}
\usepackage{listings}
\usepackage{tikz}
\usetikzlibrary{fit,positioning}
\usepackage{pstricks}
\usepackage{auto-pst-pdf}
\usepackage{pst-node}
\usepackage{comment}
\usepackage{color}
\usepackage{wrapfig,lipsum,booktabs}
\usepackage{verbatim}
\newcommand{\argmax}{\operatornamewithlimits{argmax}}

\def\bR{{\mathbb R}}

\def\NPDF{{\mathcal N}}

\def\f0{{\mathbf 0}}

\def\vect{{\operatorname{vec}}}

\newtheorem{thm}{Theorem}

\newtheorem{prop}[thm]{Proposition}

\theoremstyle{definition}

\usepackage[noadjust]{cite}

\ifCLASSINFOpdf

\else

\fi

\begin{document}

\title{Matrix Factorisation with Linear Filters}

\author{\"Omer Deniz Aky{\i}ld{\i}z% <-this % stops a space
\thanks{The author is with Bogazici University. Email: odakyildiz@gmail.com}}

\markboth{}%
{Shell \MakeLowercase{\textit{et al.}}: Bare Demo of IEEEtran.cls for Journals}

\maketitle

\begin{abstract}
This text investigates relations between two well-known family of algorithms, matrix factorisations and recursive linear filters, by describing a probabilistic model in which approximate inference corresponds to a matrix factorisation algorithm. Using the probabilistic model, we derive a matrix factorisation algorithm as a recursive linear filter. More precisely, we derive a matrix-variate recursive linear filter in order to perform efficient inference in high dimensions. We also show that it is possible to interpret our algorithm as a nontrivial stochastic gradient algorithm. Demonstrations and comparisons on an image restoration task are given.
\end{abstract}


% \fi
\IEEEpeerreviewmaketitle

\section{Introduction}
\IEEEPARstart{M}{atrix} factorisation algorithms are one of the cornerstones of modern signal processing, machine learning, and, more generally, computational linear algebra. Formally the problem can be stated as factorising a data matrix $Y\in \bR^{m\times n}$ as,
\begin{align}
Y \approx CX
\end{align}
where $C\in \bR^{m\times r}$ is the \textit{dictionary matrix}, and columns of $X \in \bR^{r\times n}$ are \textit{coefficients}, and $r$ is the approximation rank. In our setup, all matrices will be real-valued. We are interested in to solve this problem in a \textit{recursive} way, i.e. using a single data vector at each time to update factors.

This is a well-known and well-studied problem. On the matrix factorisation side, the following works are related to our work. In \cite{YildirimSMCNMF}, authors proposed a sequential Monte Carlo based nonnegative matrix factorisation (NMF) algorithm using a similar model to original probabilistic interpretation of NMF \cite{cemgil09-nmf}. The model proposed in \cite{YildirimSMCNMF} is defined over columns of $X$, and $C$ is regarded as a static but unknown variable, so estimated via maximum-likelihood techniques. In \cite{dynamicmatrixfact}, authors propose a dynamic matrix factorisation with collaborative filtering applications in mind. In \cite{sismanisSGD}, authors propose a matrix factorisation algorithm based on stochastic gradient descent (SGD). In our context, it can be applied column-wise. In \cite{mairal2010online}, authors derive an online dictionary learning algorithm which is also related to SGD but they also impose sparsity assumptions on coefficients. An approach based on recursive least squares (RLS) can be found in \cite{RLSdictlearn}. Also there is more recent work on dictionary learning based on RLS \cite{zhang2014analysis}. These RLS-based approaches factorises the dictionary matrix (e.g. $C = AB$) and update each of these factors accordingly. Moreover these RLS-based works focus on supervised learning of dictionaries whereas here we are interested in unsupervised learning and applications without training phase (such as unsupervised image restoration).

Update rules given in these papers are different than what we obtain (we do not update dictionary matrix by factoring it), and more importantly, we depart from a certain probabilistic model (instead of a \textit{cost function}) and derive the update rules as explicit inference rules. We derive matrix factorisation algorithms as approximate \textit{matrix-variate filtering} algorithms in probabilistic models. The most related works to ours are \cite{hennig2013quasi} and \cite{hennig2015probabilistic} in which authors derive matrix-variate update rules for Hessian matrices as analytic inference rules in probabilistic models to obtain quasi-Newton algorithms from a probabilistic perspective. We follow the exact same approach, and our derivations follow those works. The model defined in \cite{hennig2013quasi} and ours are slightly different as \cite{hennig2013quasi} uses a model over square and symmetric matrices, but we define the model for non-square dictionary matrices (so we re-derive the update rules), and also the model definitions are slightly different. And finally, we apply these ideas to matrix factorisation problem. The provided probabilistic characterisation opens up many possibilities for incorporating further prior knowledge, or dealing with nonstationary data in a principled way by putting dynamics on the dictionary matrix.

In the following subsection, we'll give some identities which will be very useful in proofs. In Section~\ref{SecProbModel}, we describe our generative model for matrix factorisation. In Section~\ref{SecEstAndInf}, we derive our algorithm as an estimation and inference algorithm in the probabilistic model described in the Section~\ref{SecProbModel}. In Section~\ref{SecSGDcomp}, we describe the relation between SGD based matrix factorisation, and our algorithm. In Section~\ref{SecExp}, we demonstrate our algorithm on an image restoration task. In Section~\ref{SecConc}, we conclude.
\subsection{Some useful linear algebra}
We will be heavily using the following identities from \cite{harville1997matrix}, \cite{matrixcookbook} in this paper. Let $A$ is of dimension ${m\times r}$ and $B$ is of dimension ${r \times n}$. Then the following holds,
\begin{align}\label{generalTrick}
\vect(AXB) = (B^\top \otimes A)\vect(X)
\end{align}
A particular case where this identity will be useful for us is when $\dim(A) = m\times r$ and $\dim(B) = r\times 1$. So let us note the particular case in a more useful form to us,
\begin{align}\label{vecTrick2}
(x^\top \otimes I_m) \vect(A) = \vect(Ax) = Ax.
\end{align}
where $Ax$ is also a vector, $\dim(x) = r\times 1$, and $I_m$ is $m\times m$ matrix. For a matrix $M$ where $\dim(M) = m \times r$, let $m = \vect(M)$ is a $mr\times 1$ vector. To revert this operation, we define the reshaping operator: $\vect^{-1}_{m\times r}(m) = M$. Kronecker products also have the following mixed product property,
\begin{align}\label{mixedProduct}
(A \otimes B)(C \otimes D) = (AC) \otimes (BD),
\end{align}
and the following ``inversion" property,
\begin{align}\label{invKron}
(A\otimes B)^{-1} = A^{-1} \otimes B^{-1}.
\end{align}
\section{The Probabilistic Model}\label{SecProbModel}
Let $Y \in \mathbb{R}^{m\times n}$ be the data matrix, and $C \in \mathbb{R}^{m\times r}$ and $X \in \mathbb{R}^{r\times n}$. Let us denote the $i$'th column of the data matrix $Y$ with $Y(:,i)$, and $[n] = \{1,\ldots,n\}$. The observations are generated in the following way: At time $k$, we randomly sample an index $i_k \sim [n]$. And we set $y_k = Y(:,i_k)$. So $y_k$ denotes the observation at time $k$ but not $k$'th column. Similarly, the associated column of $X$ is denoted with $x_k$, and $x_k = X(:,i_k)$. For example if $i_k = 2$, then $y_k$ would be the second column of $Y$, and $x_k$ would be the second column of $X$. We denote the dictionary matrix with $C$, and $c = \vect(C)$. This also holds for $c_k = \vect(C_k)$ where $C_k$ is $m\times r$ matrix stands for estimate of $C$ at iteration $k$.

In this work, we consider the following probabilistic model,
\begin{align}
p(c) &= \NPDF(c; c_0, V_0 \otimes I_m) \label{priorC} \\
p(y_k|c,x_k) &= \NPDF(y_k; (x_k^\top \otimes I_m) c, \lambda \otimes I_m) \label{obsVectorVal}
\end{align}
Note that $x_k$ is a static unknown model parameter vector. On the other hand, $c$ and $y_k$ are random vectors, and treated as such. To motivate the model, notice that using identity \eqref{vecTrick2} for $(x_k^\top \otimes I_m)c$, we can rewrite the likelihood \eqref{LikelihoodCform1} in the following form,
\begin{align}
p(y_k | c, x_k) &= \NPDF(y_k; C x_k, \lambda \otimes I_m). \label{LikelihoodCform1}
\end{align}
In the matrix factorisation setup, we would like to assume $y_k \approx C x_k$ for each $k$, here this corresponds to assuming Gaussian noise. Using the model \eqref{priorC} and \eqref{obsVectorVal}, we would like to estimate both $x_k$ and $C$ given the observations $y_{1:k}$, i.e. observations up to time $k$.
\section{Parameter Estimation and Inference}\label{SecEstAndInf}
From the viewpoint of probabilistic (or Bayesian) inference, coefficients $x_k$ are \textit{static} parameters to be estimated (typically by some optimisation formulation), and in contrast, the dictionary matrix $C$ is a \textit{latent} variable that is to be inferred through its posterior distribution. In this section, we'll show how to perform parameter estimation for coefficients, and inference for the dictionary matrix.
\subsection{Parameter estimation: Finding coefficients}
To estimate the parameters $x_k$ associated with a given observation $y_k$, we  formulate the following maximisation problem,
\begin{align}
x_k^* = \argmax_{x_k} p(y_k | c_{k-1},x_k)
\end{align}
Since this density is a Gaussian with mean $C_{k-1} x_k$, the solution is the pseudoinverse,
\begin{align}
x_k^* = (C_{k-1}^\top C_{k-1})^{-1} C_{k-1}^\top y_k.
\end{align}
Note that in this work, we use this update rule in the experiments. However, just to note, a very intriguing approach would be maximising the marginal likelihood $p(y_k|y_{1:k-1},x_k)$ by integrating out $c$. Unfortunately, the optimisation part is intractable and we will discuss this elsewhere\footnote{See the discussion at \url{http://almoststochastic.com}.}.
\subsection{Inference: Finding the dictionary matrix}
In this subsection, we assume $x_k$ is fixed and $x_k = x_k^*$, and we suppress $x_k$ from the notation. We consider the model \eqref{priorC} and \eqref{obsVectorVal}, and solve the posterior inference problem. We can rewrite this model in a generic way,
\begin{align*}
p(c) &= \NPDF(c;c_0,P_0), \\
p(y_k|c) &= \NPDF(y_k; H_k c, R),
\end{align*}
where $P_0 = V_0 \otimes I$ and $R = \lambda \otimes I$. Since we fix $x_k$ for all $k$, we suppress the $x_k$ from the notation, and use generic $H_k$ observation matrix which is assumed to be known now. Given this model and fixed parameters, it is well-known that given observations up to time $k$, the posterior distribution $p(c|y_{1:k})$ is Gaussian too \cite{sarkka2013bayesian}. We denote this posterior density by $p(c|y_{1:k}) = \NPDF(c;c_k,P_k)$. The mean $c_k$ and covariance $P_k$ can be found by a recursive least squares filter (recursive linear filter) algorithm. Given observations $y_{1:k}$, the mean $c_k$ is given by \cite{sarkka2013bayesian},
\begin{align*}
c_{k} =c_{k-1} + P_{k-1} H_k^\top (H_k P_{k-1} H_k^\top + R_k)^{-1} (y_k - H_k c_{k-1}),
\end{align*}
and the covariance of the posterior is given by,
\begin{align*}
P_{k} = P_{k-1} - P_{k-1} H_k^\top (H_k P_{k-1} H_k^\top + R)^{-1} H_k P_{k-1}.
\end{align*}
Implementing these update rules would be very inefficient as $c \in \bR^{mr}$ might be a very high-dimensional vector. This requires to store a huge observation matrix $H_k$ and a huge covariance matrix $P_k$ which can easily become an impractical problem to solve. But fortunately we can obtain a very efficient matrix-variate update rule using the following proposition.
\begin{prop}\label{PropMean} The posterior mean $c_k$ which is given by,
\begin{align*}
c_{k} =c_{k-1} + P_{k-1} H_k^\top (H_k P_{k-1} H_k^\top + R_k)^{-1} (y_k - H_k c_{k-1}),
\end{align*}
can be rewritten as,
\begin{align}\label{FullPosteriorMean}
C_k = C_{k-1} + \frac{(y_k - C_{k-1} x_k) x_k^\top V_{k-1}^\top}{x_k^\top V_{k-1} x_k + \lambda}.
\end{align}
\end{prop}
\begin{proof}
We put $P_{k-1} = V_{k-1} \otimes I_m$ (see Prop.~\ref{PropCov} to see this form holds for all $k$) and $H_k = x_k^\top \otimes I_m$ and $R_k = \lambda \otimes I_m$, and arrive,
\begin{align*}
c_{k} =\,\,&c_{k-1} + (V_{k-1} \otimes I_m) (x_k \otimes I_m) \\ &\left((x_k^\top \otimes I_m) (V_{k-1} \otimes I_m) (x_k \otimes I_m) + \lambda \otimes I_m \right)^{-1} \times \\&(y_k - (x_k^\top \otimes I_m) c_{k-1}),
\end{align*}
Using the mixed product property \eqref{mixedProduct} three times, using \eqref{invKron}, and finally using \eqref{vecTrick2} for the last term, one can arrive,
\begin{align*}
c_{k} = c_{k-1} + \underbrace{\left[\frac{V_{k-1} x_k}{x_k^\top V_{k-1} x_k + \lambda} \otimes I_m \right] (y_k - C_{k-1} x_k)}_{\textnormal{Use } \eqref{generalTrick}}
\end{align*}
Using \eqref{generalTrick} and reshaping with $\vect^{-1}_{m\times r}$, we obtain Eq.~\eqref{FullPosteriorMean}.
\end{proof}
One can recover classical Broyden's rule of quasi-Newton methods by setting $V_{k-1} = I$. It is already known that Broyden's rule and other quasi-Newton algorithms are recursive least squares regressors \cite{hennig2013quasi}, \cite{hennig2015probabilistic}. Thus, this is also a generalisation of a matrix factorization algorithm we proposed in our earlier work based on Broyden updates \cite{akyildiz2015online}. In the following proposition, we derive an efficient posterior covariance update to use in mean update \eqref{FullPosteriorMean}.
\begin{prop}\label{PropCov}
The posterior covariance update,
\begin{align*}
P_{k} = P_{k-1} - P_{k-1} H_k^\top (H_k P_{k-1} H_k^\top + R)^{-1} H_k P_{k-1},
\end{align*}
can be rewritten as,
\begin{align}\label{FullPosteriorCovariance}
P_{k} = \underbrace{\left(V_{k-1} - \frac{V_{k-1} x_k x_k^\top V_{k-1}}{x_k^\top V_{k-1} x_k + \lambda} \right)}_{V_{k}} \otimes I_m.
\end{align}
\end{prop}
\begin{proof} We start by putting $H_k = x_k^\top \otimes I_m$ and $R = \lambda \otimes I_m$. So we arrive,
\begin{align*}
P_{k} =&\,\, P_{k-1} - P_{k-1} (x_k\otimes I_m) ((x_k^\top \otimes I_m) P_{k-1} (x_k \otimes I_m)  + \\ 
&\lambda \otimes I_m)^{-1} (x_k^\top \otimes I_m)  P_{k-1}.
\end{align*}
We also put $P_{k-1} = V_{k-1} \otimes I_m$. We will show that this form holds also $P_k$, and since $P_0$ is also of this form, by induction, we arrive that it holds for all $k$. Let us put,
\begin{align*}
P_{k} =&\,\, (V_{k-1}\otimes I_m) - (V_{k-1}\otimes I_m) (x_k\otimes I_m) \times \\ &((x_k^\top \otimes I_m) (V_{k-1}\otimes I_m)(x_k \otimes I_m)  + \lambda \otimes I_m)^{-1} \times \\&(x_k^\top \otimes I_m) (V_{k-1}\otimes I_m).
\end{align*}
By using mixed product property \eqref{mixedProduct} several times, we obtain,
\begin{align*}
P_{k} =&\,\, (V_{k-1}\otimes I_m) - (V_{k-1}x_k \otimes I_m)\times \\ &((x_k^\top V_{k-1} x_k + \lambda)^{-1} \otimes I_m) (x_k^\top V_{k-1} \otimes I_m).
\end{align*}
where we also used property \eqref{invKron}. Few more uses of mixed product property leads to,
\begin{align*}
P_{k} =&\,\, (V_{k-1}\otimes I_m) - \frac{V_{k-1} x_k x_k^\top V_{k-1}}{x_k^\top V_{k-1} x_k + \lambda} \otimes I_m.
\end{align*}
Thus we can say that $P_k = V_k \otimes I_m$ where,
\begin{align}\label{VkUpdate}
V_k = V_{k-1} - \frac{V_{k-1} x_k x_k^\top V_{k-1}}{x_k^\top V_{k-1} x_k + \lambda}.
\end{align}
\end{proof}
We give the overall algorithm in Algorithm~\ref{MFRLF}. We name it as matrix factorisation based on recursive linear filter (MF-RLF).
\begin{algorithm}[t]
\begin{algorithmic}[1]
\caption{MF-RLF}\label{MFRLF}
\State Initialise $C_0$ randomly and set $k = 1$.
\Repeat
\State Pick $y_k = Y(:,i_k)$ where $i_k \sim [n]$ uniformly random.
\State Perform,
\begin{align*}
x_{k} &= (C_{k-1}^\top C_{k-1})^{-1} C_{k-1}^\top y_{k} \\
C_k &= C_{k-1} + \frac{(y_{k} - C_{k-1}x_{k})x_{k}^\top V_{k-1}}{\lambda + x_{k}^\top V_{k-1} x_{k}} \\
V_k &= V_{k-1} - \frac{V_{k-1} x_k x_k^\top V_{k-1}}{x_k^\top V_{k-1} x_k + \lambda}.
\end{align*}
\State $k \leftarrow k+1$
\Until{convergence}
\end{algorithmic}
\end{algorithm}
\subsection{A variation: Filtering the dictionary matrix}
We define a little modification of the model \eqref{priorC} and \eqref{obsVectorVal} and obtain a state-space model (SSM),
\begin{align*}
p(\tilde{c}_0) &= \NPDF(\tilde{c}_0; c_0,P_0) \\
p(\tilde{c}_k | \tilde{c}_{k-1}) &= \NPDF(\tilde{c}_k; \tilde{c}_{k-1}, Q_k) \\
p(y_k | \tilde{c}_k) &= \NPDF(y_k; \tilde{C}_k x_k, \lambda \otimes I_m)
\end{align*}
where now $\tilde{c}_k$ variables are latent variables, and $c_k$ is the posterior mean estimate of the $\tilde{c}_k$. All these quantities are again \textit{approximate} because all of them are conditioned on $X$ which is unknown, and to be estimated during the updates.

Deriving matrix-variate Kalman filtering recursions for this model is very similar to what we did in the previous section. Algorithmically, it is a simple modification to the Algorithm~\ref{MFRLF}. Define $Q_k = Q_V \otimes I_m$ where $Q_V$ is $r \times r$ covariance matrix. So to obtain the matrix-variate Kalman filter, it suffices to perform the following step just before step 4 of the Algorithm~\ref{MFRLF},
\begin{align*}
V_{k|k-1} = V_{k-1} + Q_V,
\end{align*}
and use $V_{k|k-1}$ for updating $C_k$ and $V_k$. We think that it could be very useful to develop an explicit model when one needs a ``forgetting" property in the dictionary. It can be a principled alternative to what is called ``forgetting factor" of the RLS when performing matrix factorisations. $Q_V$ can be actively used to add a dynamic to the dictionary matrix. We leave this potential application to the future work.
\begin{figure*}
\includegraphics[scale=0.22]{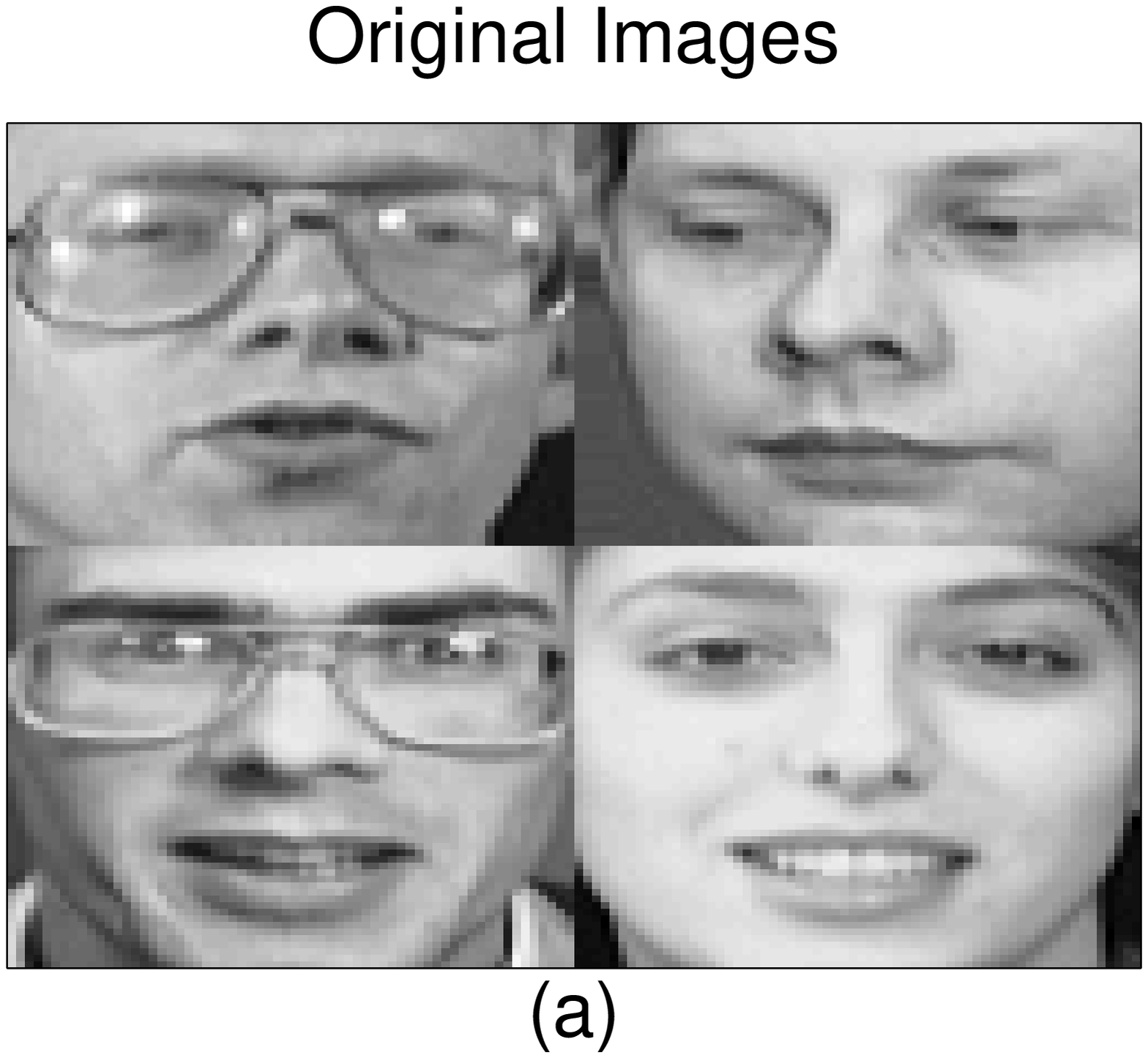}
\includegraphics[scale=0.22]{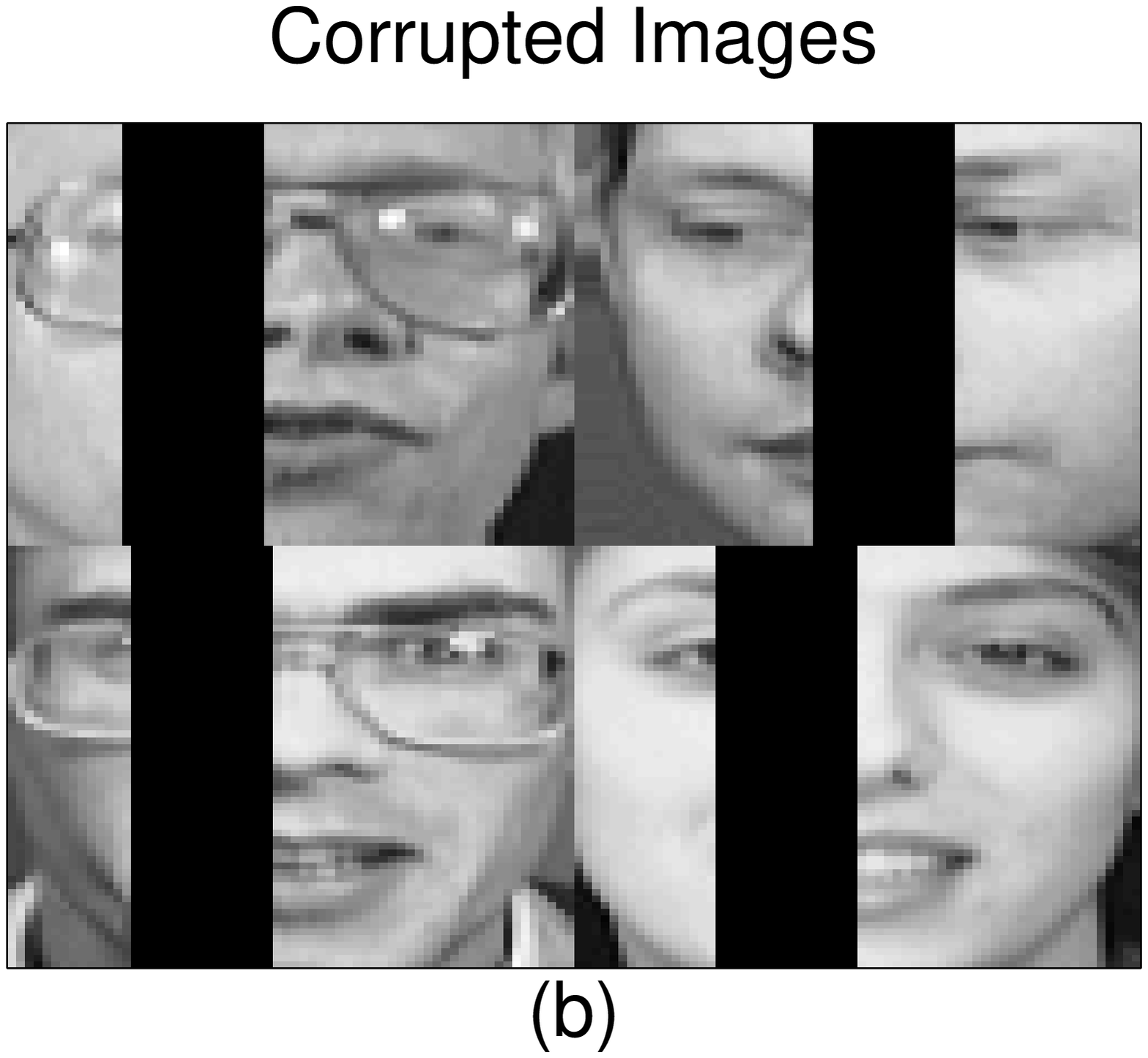}
\includegraphics[scale=0.22]{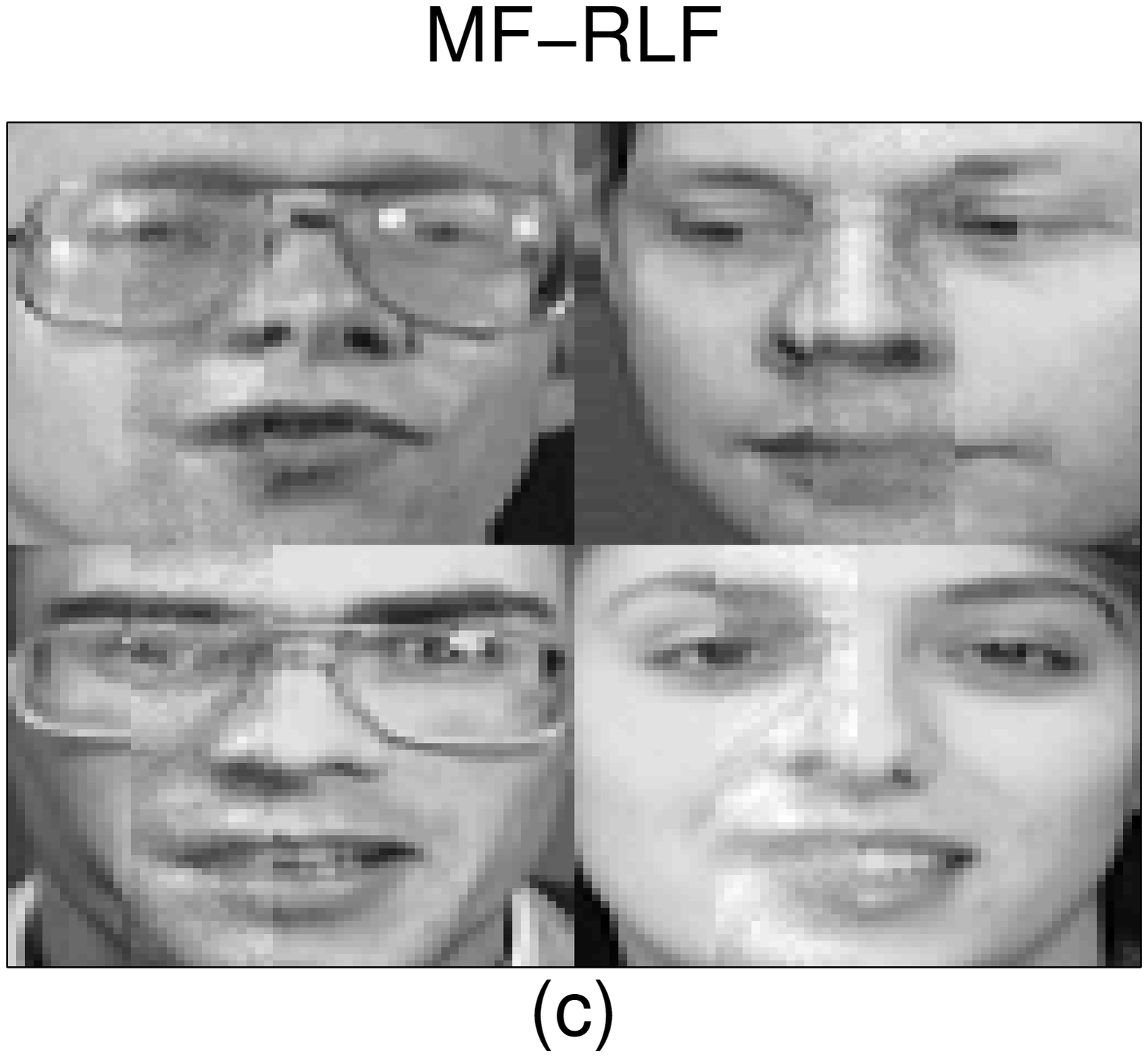}
\includegraphics[scale=0.22]{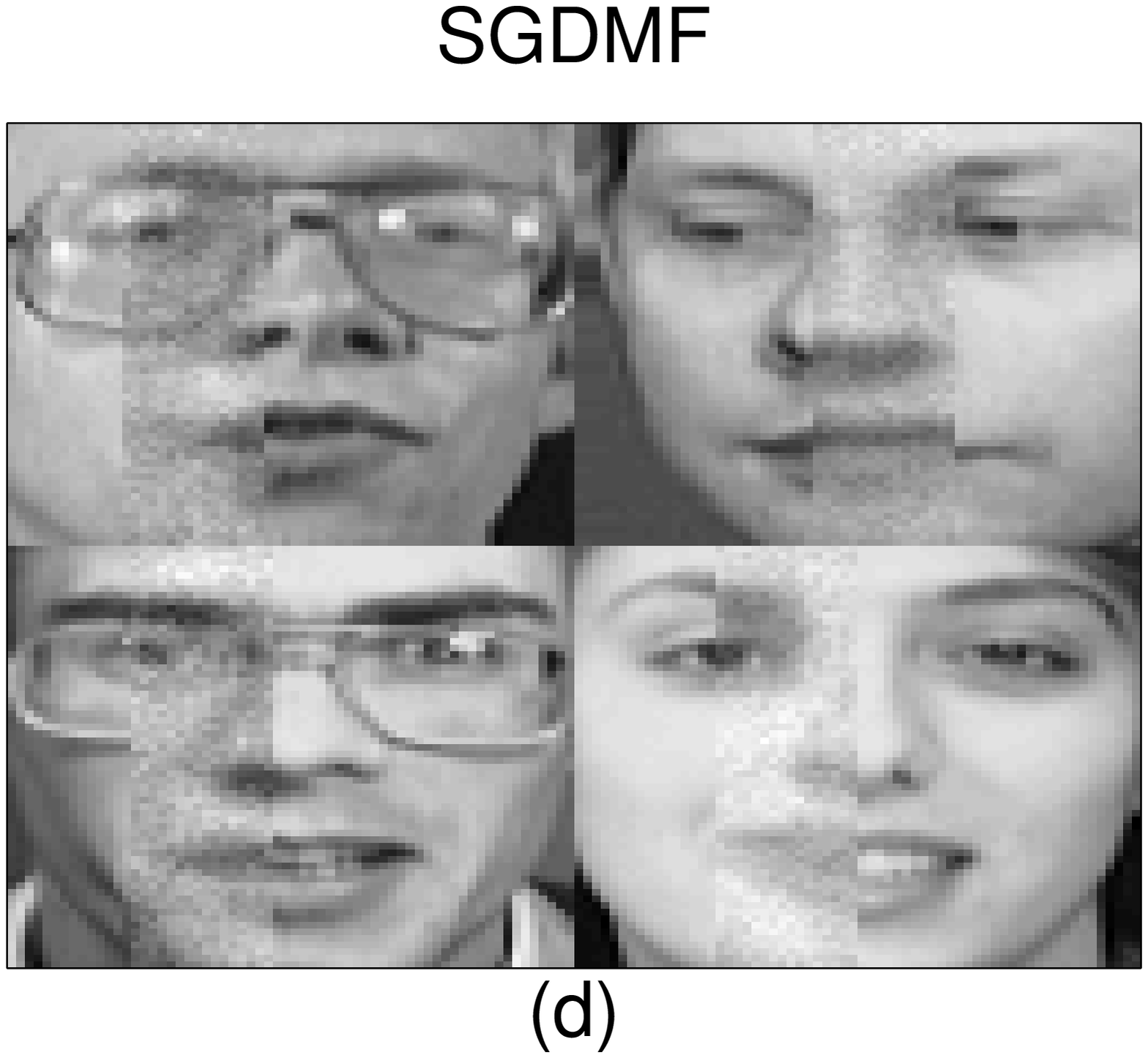}
\includegraphics[scale=0.22]{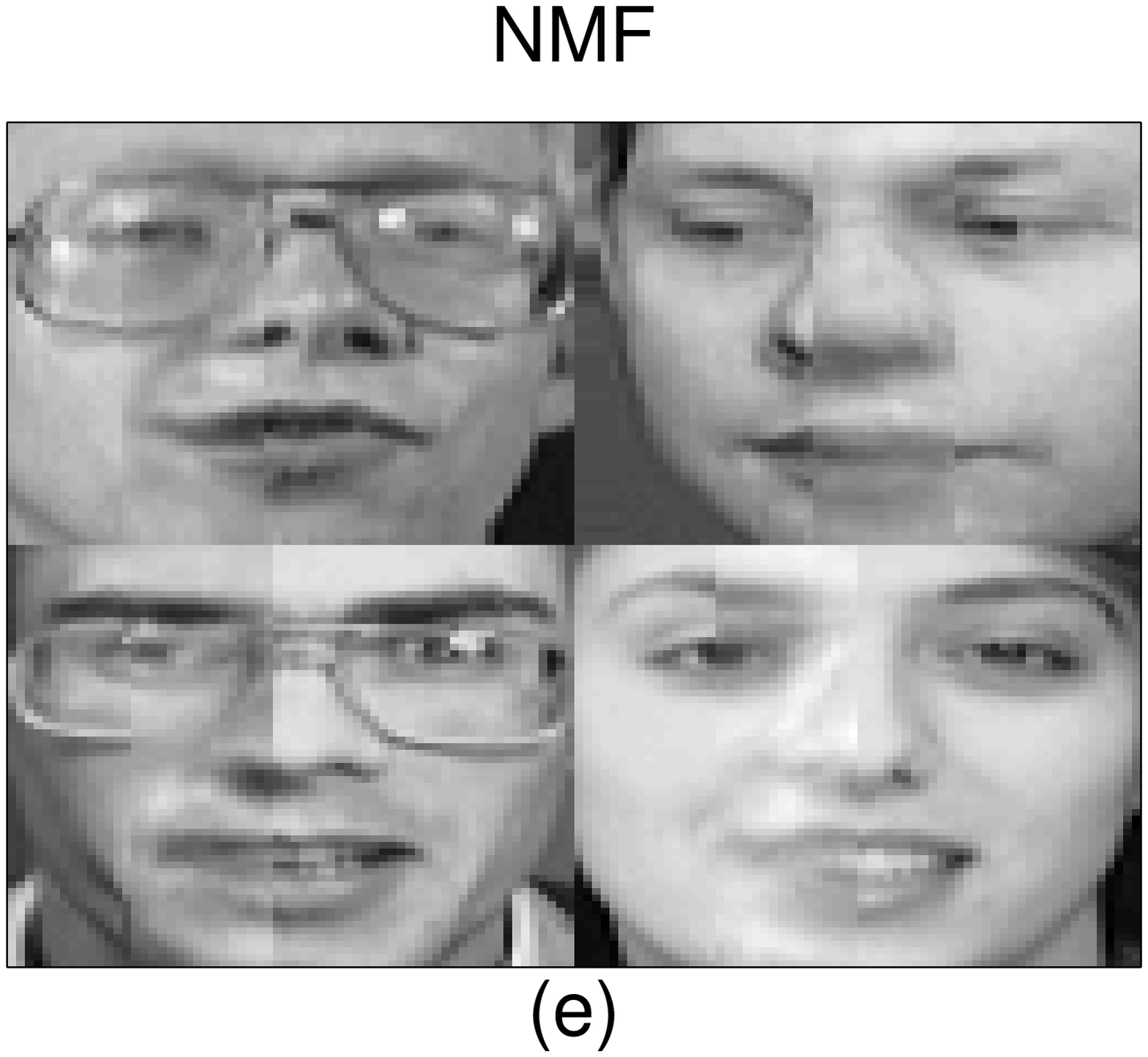}
\caption{Comparison of our algorithm with stochastic gradient descent MF (SGDMF), and nonnegative matrix factorization (NMF). SGDMF and MF-RLF passed 10 times over dataset recursively. NMF is run for 1000 batch iterations. (a) Some of original images, (b) We randomly removed \%25 batch of all columns (for all 400 faces). (c) The result of MF-RLF (Algorithm~\ref{MFRLF}). (d) Result of SGDMF. (e) Result of NMF. SNR values: MF-RLF: 12.38, NMF: 12.35, SGDMF: 11.75 where initial SNR: 0.68. This clearly shows our algorithm competes with online as well as the offline benchmark algorithms on a standard task.}
\label{MissingImFig}
\end{figure*}
\section{Relation to Stochastic Gradient Descent}\label{SecSGDcomp}
Our algorithm can be interpreted as a version of stochastic gradient descent with a nontrivial and non-scalar step size. The interpretation is given as follows: Suppose $y_k$ are iid draws conditioned on $C$ (estimating $X$ will be identical to previous case, so assume it is known), and we would like to maximise the following joint likelihood of the dataset,
\begin{align*}
p(Y | C, X) = \prod_{k=1}^n p(y_k | C, x_k),
\end{align*}
and assume this likelihood is defined as
\begin{align*}
p(y_k | C, x_k) = \NPDF(y_k; C x_k, I).
\end{align*}
Then after a bit of calculation, one can show that applying {SGD} to the negative log-likelihood results in the following iteration,
\begin{align}\label{sgdupdate}
C_k = C_{k-1} + \gamma_k (y_k - C_{k-1} x_k) x_k^\top.
\end{align}
First of all, putting $\gamma_k = 1/(\lambda + x_k^\top x_k)$ recovers Broyden's rule again from a different perspective: \textit{maximum likelihood estimation via SGD}\footnote{There are other ways, e.g. embedding step-size into the covariance. So this hints for an interesting connection between the step-size of the SGD and posterior covariance of the recursive linear-Gaussian models.}. But note that this does not ensure that the usual assumptions on the step-size is satisfied, hence the convergence is questionable \cite{Bottou98onlinelearning}. We note that, the update rule \eqref{FullPosteriorMean} that is proposed in this paper is different than \eqref{sgdupdate} as we also have a matrix $V_k$ which can not be embedded into the step-size in a trivial way.
\section{Application to Image Restoration}\label{SecExp}
We demonstrate our algorithm on an image restoration task on the Olivetti dataset \cite{cemgil09-nmf}. This dataset consists of $400$ face images of size $64\times 64$. We vectorise each face into a column vector with dimension $4096$, so $m = 4096$ in this problem. Since there are $400$ faces in the dataset, $n = 400$. We chose $r = 40$ as an approximation rank and $\lambda = 2$.  \begin{wraptable}{r}{3cm}
\caption{Table of SNR Values. Initial SNR: 0.68}\label{tableSnr}
\begin{tabular}{| l | l |}
    \hline
    Algorithm & SNR \\ \hline
    MF-RLF & 12.38 \\ \hline
    SGDMF & 11.75 \\ \hline
    NMF & 12.35 \\
    \hline
    \end{tabular} 
\end{wraptable}We initialised factors randomly without imposing any structure. We choose $V_0 = I$ for this particular dataset, other choices lead to poorer performance. But it is entirely up to user to encode a prior knowledge about dictionary by using covariance matrix $V_0$ that encodes a qualitative knowledge about the structure between $r$ columns of the dictionary matrix.

We deal with missing data using exact same methodology described in \cite{akyildiz2015online}. So we define a mask $M$, and denote the mask associated with $y_k$ with $m_k$. So in the Algorithm~\ref{MFRLF}, we replace $(y_k - C_{k-1} x_k)$ term by $m_k \odot (y_k - C_{k-1} x_k)$. Also while updating $C_k$, we construct a special mask,
\begin{align*}
M_{C_k} = \underbrace{[m_{k},\ldots, m_{k}]}_{r \textnormal{ times}},
\end{align*}
and apply the following update,
\begin{align*}
x_{k} =& ((M_{C_{k-1}} \odot C_{k-1})^\top (M_{C_{k-1}} \odot C_{k-1}))^{-1} \times \\
&(M_{C_{k-1}} \odot C_{k-1})^\top (m_{{k}}\odot y_{{k}}),
\end{align*}
in the Algorithm~\ref{MFRLF} for $x_k$. Note that all these reformulations can be derived from the model by putting masks into the model. We left them out for simplicity. For the more details of this missing data handling scheme see \cite{akyildiz2015online}. We use this both for SGD and MF-RLF.

We give comparisons with both {SGD} and {NMF}. We give a comparison with {NMF} because we think that \textit{the most basic task} of an online algorithm is to compete with the state-of-the-art batch methods. In general, many online algorithms fail at fulfilling this task because datasets which one can experiment batch algorithms are too small for online learning. In this section, we show that our algorithm fulfils this \textit{hard} task: It works as good as NMF --the standard batch benchmark-- on image restoration. As our algorithm bears some similarities with {SGD}, we also give a comparison with {SGD} as an online algorithm. The implementation is similar to ours -- the $C_k$ update \eqref{sgdupdate} subsequently followed by pseudoinverse. The visual results can be seen from Fig.~\ref{MissingImFig}, and SNR values are tabulated in Table~\ref{tableSnr}. MF-RLF and SGDMF passed recursively 10 times over the dataset, i.e. using a single observation each iteration. We ran NMF with 1000 batch passes over data. This shows these recursive algorithms uses data much more efficiently.

Results show that our algorithm works well perceptually, and achieve same SNR values with NMF although it only passed 10 times over the dataset.
\section{Conclusion}\label{SecConc}
We presented a matrix factorisation algorithm which makes use of linear filters. We recast the factorisation problem into the linear filtering problem, and propose efficient matrix-variate update rules for the Gaussian posterior summary statistics. The algorithm can trivially be extended for dynamic models on dictionary matrix where one can model changing nature of the dataset in a principled way. For the future work, we think to extend this filtering approach to nonlinear and non-Gaussian state space models where the model structure can be much more richer than linear models. Putting a nonlinear dynamics on $x_k$ poses new challenges for sequential inference schemes in high-dimensions, and calls for Rao-Blackwellisation of state-of-the-art algorithms (such as \cite{thomasNested}) proposed for high-dimensional filtering. Another potential use of our algorithm can be based on uncertainty estimates: Covariance uncertainty can be used to stop unnecessary computations, and save enormous time in a related fashion to probabilistic numerics \cite{hennig2015probabilisticnumerics}. We hope to pursue different methodological and application based directions for the future.
\section*{Acknowledgements}
I am grateful to Philipp Hennig for very valuable discussions. I am thankful to A. Taylan Cemgil for his support and suggestions. I am thankful to Baris Evrim Demiroz and Thomas Sch\"on for discussions.  This work is supported by TUBITAK under the grant number 113M492 (PAVERA).
\ifCLASSOPTIONcaptionsoff
  \newpage
\fi
\bibliographystyle{IEEEtran}
\bibliography{../../../../CommonLatexFiles/draft}

\begin{thebibliography}{10}
\providecommand{\url}[1]{#1}
\csname url@samestyle\endcsname
\providecommand{\newblock}{\relax}
\providecommand{\bibinfo}[2]{#2}
\providecommand{\BIBentrySTDinterwordspacing}{\spaceskip=0pt\relax}
\providecommand{\BIBentryALTinterwordstretchfactor}{4}
\providecommand{\BIBentryALTinterwordspacing}{\spaceskip=\fontdimen2\font plus
\BIBentryALTinterwordstretchfactor\fontdimen3\font minus
  \fontdimen4\font\relax}
\providecommand{\BIBforeignlanguage}[2]{{%
\expandafter\ifx\csname l@#1\endcsname\relax
\typeout{** WARNING: IEEEtran.bst: No hyphenation pattern has been}%
\typeout{** loaded for the language `#1'. Using the pattern for}%
\typeout{** the default language instead.}%
\else
\language=\csname l@#1\endcsname
\fi
#2}}
\providecommand{\BIBdecl}{\relax}
\BIBdecl

\bibitem{YildirimSMCNMF}
S.~Yildirim, A.~T. Cemgil, and S.~S. Singh, ``An online
  expectation-maximisation algorithm for nonnegative matrix factorisation
  models,'' in \emph{16th IFAC Symposium on System Identification (SYSID
  2012)}, 2012.

\bibitem{cemgil09-nmf}
A.~T. Cemgil, ``Bayesian inference in non-negative matrix factorisation
  models,'' \emph{Computational Intelligence and Neuroscience}, no. Article ID
  785152, 2009.

\bibitem{dynamicmatrixfact}
J.~Z. Sun, K.~R. Varshney, and K.~Subbian, ``Dynamic matrix factorization: A
  state space approach,'' in \emph{Acoustics, Speech and Signal Processing
  (ICASSP), 2012 IEEE International Conference on}.\hskip 1em plus 0.5em minus
  0.4em\relax IEEE, 2012, pp. 1897--1900.

\bibitem{sismanisSGD}
R.~Gemulla, E.~Nijkamp, P.~J. Haas, and Y.~Sismanis, ``Large-scale matrix
  factorization with distributed stochastic gradient descent,'' in
  \emph{Proceedings of the 17th ACM SIGKDD international conference on
  Knowledge discovery and data mining}.\hskip 1em plus 0.5em minus 0.4em\relax
  ACM, 2011, pp. 69--77.

\bibitem{mairal2010online}
J.~Mairal, F.~Bach, J.~Ponce, and G.~Sapiro, ``Online learning for matrix
  factorization and sparse coding,'' \emph{The Journal of Machine Learning
  Research}, vol.~11, pp. 19--60, 2010.

\bibitem{RLSdictlearn}
K.~Skretting and K.~Engan, ``Recursive least squares dictionary learning
  algorithm,'' \emph{Signal Processing, IEEE Transactions on}, vol.~58, no.~4,
  pp. 2121--2130, 2010.

\bibitem{zhang2014analysis}
Y.~Zhang, H.~Wang, and W.~Wang, ``An analysis dictionary learning algorithm
  based on recursive least squares,'' in \emph{Signal Processing (ICSP), 2014
  12th International Conference on}.\hskip 1em plus 0.5em minus 0.4em\relax
  IEEE, 2014, pp. 831--835.

\bibitem{hennig2013quasi}
P.~Hennig and M.~Kiefel, ``Quasi-newton methods: A new direction,'' \emph{The
  Journal of Machine Learning Research}, vol.~14, no.~1, pp. 843--865, 2013.

\bibitem{hennig2015probabilistic}
P.~Hennig, ``Probabilistic interpretation of linear solvers,'' \emph{SIAM
  Journal on Optimization}, vol.~25, no.~1, pp. 234--260, 2015.

\bibitem{harville1997matrix}
D.~A. Harville, \emph{Matrix algebra from a statistician's perspective}.\hskip
  1em plus 0.5em minus 0.4em\relax Springer, 1997, vol.~1.

\bibitem{matrixcookbook}
K.~B. Petersen, M.~S. Pedersen \emph{et~al.}, ``The matrix cookbook,''
  \emph{Technical University of Denmark}, vol.~7, p.~15, 2008.

\bibitem{sarkka2013bayesian}
S.~S{\"a}rkk{\"a}, \emph{Bayesian filtering and smoothing}.\hskip 1em plus
  0.5em minus 0.4em\relax Cambridge University Press, 2013, no.~3.

\bibitem{akyildiz2015online}
{\"O}.~D. Aky{\i}ld{\i}z, ``Online matrix factorization via broyden updates,''
  \emph{arXiv preprint, arXiv:1506.04389}, 2015.

\bibitem{Bottou98onlinelearning}
L.~Bottou, ``Online learning and stochastic approximations,'' 1998.

\bibitem{thomasNested}
C.~A. Naesseth, F.~Lindsten, and T.~B. Schon, ``Nested sequential monte carlo
  methods,'' in \emph{In Proceedings of the 32 nd International Conference on
  Machine Learning, Lille, France}, 2015.

\bibitem{hennig2015probabilisticnumerics}
P.~Hennig, M.~A. Osborne, and M.~Girolami, ``Probabilistic numerics and
  uncertainty in computations,'' in \emph{Proc. R. Soc. A}, vol. 471, no.
  2179.\hskip 1em plus 0.5em minus 0.4em\relax The Royal Society, 2015, p.
  20150142.

\end{thebibliography}
\end{document}